\documentclass{article}
\usepackage{kr}

\usepackage{times}
\usepackage{soul}
\usepackage{url}
\usepackage[hidelinks]{hyperref}
\usepackage[utf8]{inputenc}
\usepackage[small]{caption}
\usepackage{graphicx}
\usepackage{amsmath}
\usepackage{amsthm}
\usepackage{booktabs}
\usepackage{algorithm}
\usepackage{algorithmic}

\usepackage{multirow}
\usepackage{enumitem}
\usepackage{savesym}
\savesymbol{Bbbk}
\usepackage{amssymb}
\restoresymbol{NEW}{Bbbk}
\usepackage{xcolor}

\usepackage{balance} 

\newcommand{\FTn}[1]{\textcolor{black}{#1}}

\newcommand{\BIn}[1]{\textcolor{black}{#1}} 

\newcommand{\BI}[1]{\textcolor{black}{#1}} 

\newcommand{\FT}[1]{\textcolor{black}{#1}} 
\newcommand{\AR}[1]{\textcolor{black}{#1}}

\newcommand{\AX}[1]{\textcolor{black}{#1}}

\newtheorem{proposition}{Proposition}

\def\resp{resp.}
\def\Definition{Def.}

\def\Question{\ensuremath{\mathcal{Q}}}
\def\SF{\ensuremath{\mathcal{\sigma}}}

\def\Forecast{\ensuremath{\mathcal{F}}}
\def\Update{\ensuremath{\mathcal{U}}}
\def\Time{\ensuremath{\mathcal{T}}}
\def\Proposal{\ensuremath{\mathcal{P}}}
\def\Agents{\ensuremath{\mathcal{A}}}
\def\Votes{\ensuremath{\mathcal{V}}}
\def\Confidence{\ensuremath{\mathcal{C}}}
\def\Brier{\ensuremath{\mathcal{B}}}
\def\Number{\ensuremath{\mathcal{N}}}
\def\Outcome{\ensuremath{\mathcal{O}}}
\def\NewForecast{\ensuremath{\Forecast^\Proposal}}

\def\GroupForecast{\ensuremath{\Forecast^g}}
\def\AgentForecast{\ensuremath{\Forecast^\Agents}}
\def\AmmArgs{\ensuremath{\mathcal{X}}}
\def\IncArgs{\ensuremath{\AmmArgs^\uparrow}}
\def\DecArgs{\ensuremath{\AmmArgs^\downarrow}}
\def\AttArgs{\ensuremath{\AmmArgs^-}}
\def\SuppArgs{\ensuremath{\AmmArgs^+}}
\def\BS{\ensuremath{\mathcal{\tau}}}
\def\Rels{\ensuremath{\mathcal{R}}}
\def\RelsP{\ensuremath{\Rels^p}}
\def\RelsQ{\ensuremath{\Rels}}

\def\update{\ensuremath{u}}
\def\agent{\ensuremath{a}}
\def\arg{\ensuremath{x}}
\def\Brier{\ensuremath{b}}
\def\decarg{\ensuremath{\arg^\downarrow}}
\def\decarga{\ensuremath{\decarg_1}}
\def\decargb{\ensuremath{\decarg_2}}
\def\incarg{\ensuremath{\arg^\uparrow}}
\def\incarga{\ensuremath{\incarg_1}}
\def\incargb{\ensuremath{\incarg_2}}
\def\incargc{\ensuremath{\incarg_3}}
\def\supparg{\ensuremath{\arg^+}}
\def\supparga{\ensuremath{\supparg_1}}
\def\suppargb{\ensuremath{\supparg_2}}
\def\attarg{\ensuremath{\arg^-}}
\def\attarga{\ensuremath{\attarg_1}}
\def\attargb{\ensuremath{\attarg_2}}
\def\attargc{\ensuremath{\attarg_3}}

\newtheorem{example}{Example}

\newtheorem{definition}{Definition}

\title{Forecasting Argumentation Frameworks}

\author{%
Benjamin Irwin\and
Antonio Rago\And
Francesca Toni\\
\affiliations
Imperial College London\\
\emails
\{benjamin.irwin19, a.rago, ft\}@imperial.ac.uk
}

\begin{document}

\maketitle

\begin{abstract}
We introduce \emph{Forecasting Argumentation Frameworks (FAFs)}, a novel  argumentation-based methodology for forecasting informed by recent judgmental forecasting research. FAFs comprise 
update frameworks which empower (human or artificial) agents to argue over time 
about the probability of outcomes, e.g. the winner of a political election or a fluctuation in inflation rates, whilst flagging perceived \emph{irrationality} in the  agents' behaviour with a view to improving their forecasting accuracy. FAFs include five 
argument types,
amounting to standard  \emph{pro/con} arguments, as in bipolar argumentation, as well as novel  \emph{proposal} arguments and \emph{increase}/\emph{decrease amendment} arguments.
We adapt an existing gradual semantics for bipolar argumentation to determine the aggregated dialectical strength of proposal arguments and
define irrational behaviour
. We then give a simple aggregation function which produces a final group forecast from rational agents'  individual forecasts. We identify and study properties of FAFs and conduct an empirical evaluation which signals FAFs' potential to increase the forecasting accuracy of participants.
\end{abstract}

\section{Introduction}\label{sec:intro}


Historically, humans have performed inconsistently in judgemental forecasting \cite{Makridakis2010,TetlockExp2017}, which incorporates subjective opinion and probability estimates to predictions \cite{Lawrence2006}. Yet, human judgement remains essential in cases where pure statistical methods are not applicable, e.g. where historical data alone is insufficient or for one-off, more `unknowable' events \cite{Petropoulos2016,Arvan2019,deBaets2020}. Judgemental forecasting is widely relied upon for decision-making \cite{Nikolopoulos2021}, in myriad fields from epidemiology to national security \cite{Nikolopoulos2015,Litsiou2019}. Effective tools to help humans improve their predictive capabilities thus have enormous potential for impact. Two recent global events -- the COVID-19 pandemic and the US withdrawal from Afghanistan -- underscore this by highlighting the human and financial cost of predictive deficiency. A multi-purpose system which could improve our ability to predict the incidence and impact of events 
by as little as 5\%, could save millions of lives and be worth trillions of dollars per year \cite{TetlockGard2016}.

Research on judgemental forecasting (see \cite{Lawrence2006,Zellner2021} for overviews), including the recent\AX{,} 
groundbreaking `Superforecasting Experiment' \cite{TetlockGard2016}, is instructive in establishing the desired properties for systems for supporting forecasting
.
In addition to reaffirming the importance of fine-grained probabilistic reasoning \cite{Mellers2015}, this literature points to the benefits of some group techniques versus solo forecasting \cite{Landeta2011,Tetlock2014art}, of synthesising qualitative and quantitative information \cite{Lawrence2006}, of combating agents' irrationality \cite{Chang2016} and of high agent engagement with the forecasting challenge, e.g. robust debating \cite{Landeta2011} and frequent prediction updates \cite{Mellers2015}.

Meanwhile, \emph{computational argumentation}  (see \cite{AImagazine17,handbook} for recent overviews) is a field of AI which involves reasoning with uncertainty and resolving conflicting information, e.g. in natural language debates. As such, it is an ideal candidate for aggregating the broad, polymorphous set of information involved in judgemental group forecasting. An extensive and growing literature is based on various argumentation frameworks -- rule-based systems for aggregating, representing and evaluating sets of arguments, such as those applied in the contexts of \emph{scheduling} \cite{Cyras_19}, \emph{fact checking} \cite{Kotonya_20} or in various instances of \emph{explainable AI} \cite{Cyras_21}. 
Subsets of the requirements for forecasting systems are addressed by individual formalisms, e.g. \emph{probabilistic argumentation}
\AX{\cite{Dung2010,Thimm2012,Hunter2013,Fazzinga2018}} 
may effectively represent and analyse uncertain arguments about the future. However, we posit that a purpose-built argumentation framework for forecasting is essential to effectively utilise 
argumentation's reasoning capabilities in this context.

\begin{figure*}
  \includegraphics[width=\textwidth]{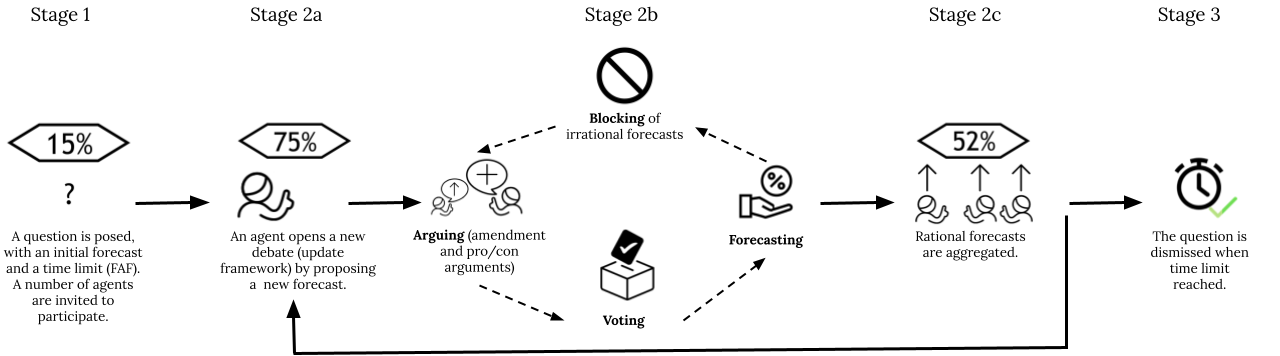}
  \caption{The step-by-step process of a FAF over its lifetime.}
  \label{fig:FAFdiag}
\end{figure*}

In this paper, we attempt to cross-fertilise these two as of yet unconnected academic areas. We draw from forecasting literature to inform the design of a new computational argumentation approach: \emph{Forecasting Argumentation Frameworks} (FAFs). FAFs 
empower (human and artificial) agents to structure  debates in real time and to deliver argumentation-based forecasting.
They offer an approach in the spirit of \emph{deliberative democracy} \cite{Bessette1980} to respond to a forecasting problem over time. The steps which underpin FAFs are depicted in Figure \ref{fig:FAFdiag} (referenced throughout) and can be described in simple terms \FT{as follows}: a FAF is initialised with a time limit \FT{(for the overall forecasting process and for each iteration therein)} and a pre-agreed `base-rate' forecast $\Forecast$ (Stage 1), e.g. based on historical data. 
\FT{Then,} the forecast is revised by one or more (non-concurrent) debates, \BI{in the form of `update frameworks' (Stage 2)}, opened and resolved by participating agents \FT{(}until 
\FT{the} specified time limit is reached\FT{)}. Each update framework begins with a proposed revision to the current forecast (Stage 2a), and proceeds with a cycle of argumentation (Stage 2b) about the proposed forecast, voting on said argumentation and forecasting. Forecasts deemed `irrational' with a view to agents' argumentation and voting are blocked. Finally, the rational forecasts are aggregated and the result replaces the current group forecast (Stage 2c). This process may be repeated over time \BI{in an indefinite number of update frameworks} (thus continually \BI{revising} the group forecast) until the 
\FT{(overall)} time limit is reached. 
The composite nature of this process enables the appraisal of new 
information relevant to the forecasting question as and when it arrives. Rather than confronting an unbounded forecasting question with a diffuse set of possible debates open at once, all agents concentrate their argumentation on a single topic (a proposal) at any given time.


After giving the necessary background on forecasting and argumentation (§\ref{sec:background}), we formalise our \FT{update} framework\FT{s for Step 2a} 
(§\ref{sec:fw}).
We then give 
\FT{our} notion of rationality \FT{(Step 2b)}, along with \FT{our} new method for 
\FT{aggregating rational forecasts (Step 2c)} from a group of agents (§\ref{sec:forecasting}) \FT{and FAFs overall}. We explore the underlying properties of 
\FT{FAFs} (§\ref{sec:props}), before describing 
\FT{\AX{an }
experiment} 
with 
\FT{a prototype implementing} our approach (§\ref{sec:experiments}). Finally, we conclude and suggest potentially fruitful avenues for future work (§\ref{sec:conclusions}).

\section{Background}\label{sec:background}

\subsection{Forecasting}

Studies on the efficacy of judgemental forecasting have shown mixed results \cite{Makridakis2010,TetlockExp2017,Goodwin2019}. Limitations of the judgemental approach are a result of well-documented cognitive biases \cite{Kahneman2012}, irrationalities in human probabilistic reasoning which lead to distortion of forecasts. Manifold methodologies have been explored to improve judgemental forecasting accuracy to varying success \cite{Lawrence2006}. These methodologies include, but are not limited to, prediction intervals \cite{Lawrence1989}, decomposition \cite{MacGregorDonaldG1994Jdwd}, structured analogies \cite{Green2007,Nikolopoulos2015} and unaided judgement \cite{Litsiou2019}. Various group forecasting techniques have also been explored \cite{Linstone1975,Delbecq1986,Landeta2011}, although the risks of groupthink \cite{McNees1987} and the importance of maintaining the independence of each group member's individual forecast are well established \cite{Armstrong2001}. Recent advances in the field have been led by Tetlock and Mellers' superforecasting experiment \cite{TetlockGard2016}, which leveraged \AX{geopolitical} forecasting tournaments and a base of 5000 volunteer forecasters to identify individuals with consistently exceptional accuracy (top 2\%). The experiment\AR{'s} findings underline the effectiveness of group forecasting orientated around debating \cite{Tetlock2014art}, and demonstrate a specific cognitive-intellectual approach conducive for forecasting \cite{Mellers20151,Mellers2015}, but stop short of suggesting a concrete universal methodology for higher accuracy. Instead, Tetlock draws on his own work and previous literature to crystallise a broad set of methodological principles by which superforecasters abide \cite[pg.144]{TetlockGard2016}:

\begin{itemize}
    \item \emph{Pragmatic}: not wedded to any idea or agenda;
    \item \emph{Analytical:} capable of stepping back from the tip-of-your-nose perspective 
    and considering other views;
    \item \emph{Dragonfly-eyed:} value diverse views and synthesise them into their own;
    \item \emph{Probabilistic:} judge using many grades of maybe;
    \item \emph{Thoughtful updaters:} when facts change, they change their minds;
    \item \emph{Good intuitive psychologists:} aware of the value of checking thinking for cognitive and emotional biases.
\end{itemize}

Subsequent research after the superforecasting experiment has included exploring further optimal forecasting tournament preparation \cite{penn_global_2021,Katsagounos2021} and extending Tetlock and Mellers' approach to answer broader, more time-distant questions \cite{georgetown}. It should be noted that there have been no recent advances on computational tool\AX{kits} 
for the field similar to that proposed in this paper.

\subsection{Computational Argumentation}

We posit that existing argumentation formalisms are not well suited for the aforementioned future-based arguments, which are necessarily semantically and structurally different from arguments about present or past concerns. Specifically, forecasting arguments are inherently probabilistic and must deal with the passage of time and its implications for the outcomes at hand.
Further, several other important characteristics can be drawn from the forecasting literature which render current argumentation formalisms unsuitable, e.g. the paramountcy of dealing with bias (in data and cognitive), forming granular conclusions, fostering group debate and the co-occurrence of qualitative and quantitative arguing.

Nonetheless,
several of these characteristics have been previously explored in argumentation and our formalisation draws from several existing approaches
. First and foremost, it draws in spirit from abstract argumentation frameworks (AAFs) \cite{Dung1995}, in that 
the arguments' inner contents are ignored and the focus is on the relationships between arguments. 
However, we consider arguments of different types and \AX{an additional relation of} support (pro), 
\AX{rather than} attack (con) alone as in \cite{Dung1995}. 
%
 Past work has also introduced probabilistic constraints into argumentation frameworks. {Probabilistic AAFs} (prAAFs) propose two divergent ways for modelling uncertainty in abstract argumentation using 
 probabilities 
 - the constellation approach \cite{Dung2010,Li2012} and the epistemic approach \cite{Hunter2013,Hunter2014,Hunter2020}. 
These formalisations use probability as a means to assess uncertainty over the validity of arguments (epistemic) or graph topology (constellation), but do not enable reasoning \emph{with} or \emph{about} probability, which is fundamental in forecasting. 
In exploring temporality, \cite{Cobo2010} augment AAFs by providing each argument with a limited lifetime. Temporal constraints have been extended in \cite{Cobo2012} and \cite{Baron2014}. Elsewhere, \cite{Rago2017} have used argumentation to model irrationality or bias in 
agents. Finally, a wide range of gradual evaluation methods have gone beyond traditional qualitative semantics by measuring arguments' acceptability on a scale (normally 
[0,1]) \cite{Leite2011,Evripidou2012,Amgoud2017,Amgoud2018,Amgoud2016}. 
Many of these approaches have been unified as Quantitative Bipolar Argumentation Frameworks (QBAFs) in \cite{Baroni2018}.

Amongst existing approaches, of special relevance in this paper are Quantitative Argumentation Debate (QuAD) frameworks \cite{Baroni2015}, 
i.e. 5-tuples ⟨$\mathcal{X}^a$, $\mathcal{X}^c$, $\mathcal{X}^p$, $\mathcal{R}$, $\BS$⟩ 
where
$\mathcal{X}^a$ is a finite set of \emph{answer} arguments (to implicit \emph{issues}); $\mathcal{X}^c$ is a finite set of \emph{con} arguments; 
$\mathcal{X}^p$ is a finite set of \emph{pro} arguments; 
$\mathcal{X}^a$, $\mathcal{X}^c$ and $\mathcal{X}^p$ are pairwise disjoint; $\mathcal{R} \subseteq (\mathcal{X}^c \cup \mathcal{X}^p) \times (\mathcal{X}^a \cup \mathcal{X}^c \cup \mathcal{X}^p)$ is an acyclic binary relation; 
$\BS$ : $(\mathcal{X}^a \cup \mathcal{X}^c \cup \mathcal{X}^p) \rightarrow 
[0,1]$ is a total function: $\BS(a)$ is the \emph{base score} of $a$. 
Here, attackers and supporters of arguments are determined by the pro and con arguments they are in relation with. Formally, for any $a\in\mathcal{X}^a \cup \mathcal{X}^c \cup \mathcal{X}^p$, the set of \emph{con} (\emph{pro}\AX{)}  \emph{arguments} of $a$ is $\mathcal{R}^-(a) = \{b\in\mathcal{X}^c|(b,a)\in\mathcal{R}\}$ 
($\mathcal{R}^+(a) = \{b\in\mathcal{X}^p|(b,a)\in\mathcal{R}\}$, \resp).
Arguments in QuAD frameworks are scored by
the \emph{Discontinuity-Free QuAD} (DF-QuAD) algorithm \cite{Rago2016},  using the argument's intrinsic base score and the \emph{strengths} of its pro/con arguments. \FTn{Given that DF-QuAD is used to define our method (see \Definition~\ref{def:conscore}), for completeness we define it formally here.} DF-QuAD's \emph{strength aggregation function} 
is defined as $\Sigma : 
[0,1]^* \rightarrow 
[0,1]$, where for $\mathcal{S} = (v_1,\ldots,v_n) \in 
[0,1]^*$:
if $n=0$, $\Sigma(S) = 0$; 
if $n=1$, $\Sigma(S) = v_1$; 
if $n=2$, $\Sigma(S) = f(v_1, v_2)$; 
if $n>2$, $\Sigma(S) = f(\Sigma(v_1,\ldots,v_{n-1}), v_n)$; 
with the \emph{base function} $f:
[0,1]\times [0,1] \rightarrow [0,1]$ defined, for $v_1, v_2\in
[0,1]$, as:
$f(v_1,v_2)=v_1+(1-v_1)\cdot v_2 = v_1 + v_2 - v_1\cdot v_2$.
After separate aggregation of the argument's pro/con descendants, the combination function $c : 
[0,1]\times
[0,1]\times
[0,1]\rightarrow
[0,1]$ combines $v^-$ and $v^+$ with the argument's base score ($v^0$):
$c(v^0,v^-,v^+)=v^0-v^0\cdot\mid v^+ - v^-\mid\:if\:v^-\geq v^+$ and
$c(v^0,v^-,v^+)=v^0+(1-v^0)\cdot\mid v^+ - v^-\mid\:if\:v^-< v^+$, \resp\
The inputs for the combination function are provided by the \emph{score function}, $\SF : \mathcal{X}^a\cup\mathcal{X}^c\cup\mathcal{X}^p\rightarrow
[0,1]$, which gives the argument's strength, as follows: for any $\arg \in \mathcal{X}^a\cup\mathcal{X}^c\cup\mathcal{X}^p$:
$\SF(\arg) = c(\BS(\arg),\Sigma(\SF(\mathcal{R}^-(\arg))),\Sigma(\SF(\mathcal{R}^+(\arg))))$
where if $(\arg_1,\ldots,\arg_n)$ is an arbitrary permutation of the ($n \geq 0$) con arguments in $\mathcal{R}^-(\arg)$, $\SF(\mathcal{R}^-(\arg))=(\SF(\arg_1),\ldots,\SF(\arg_n))$ (similarly for pro arguments).
Note that the DF-QuAD notion of $\SF$ can be applied to any argumentation framework where arguments are equipped with base scores and pro/con arguments. We will do so later, for our novel formalism.

\section{
Update \AX{F}rameworks}\label{sec:fw}
We begin by defining the individual components of our frameworks, starting with the fundamental notion of a 
\emph{forecast}. 
\FT{This} is a probability estimate for the positive outcome of a given (binary) question. 

\begin{definition}
A \emph{forecast} $\Forecast$ is the probability $P(\Question=true) \in [0,1]$ for a given \emph{forecasting question} $\Question$. 
\end{definition}


\begin{example}
\label{FAFEx}
Consider the forecasting question $\Question$: \emph{`Will the Tokyo \AX{2020 Summer} Olympics be cancelled/postponed to another year?'}.
\AX{Here, the $true$ outcome amounts to the Olympics being cancelled/postponed (and $false$ to it taking place in 2020 as planned).}
Then, a forecast $\Forecast$ may be $P(\Question=true)= 0.15$\, which amounts to a 15\% probability of the Olympics \BIn{being cancelled/postponed}. \BI{Note that $\Forecast$ may have been introduced as part of an update framework (herein described), or as an initial base rate at the outset of a FAF (Stage 1 in Figure \ref{fig:FAFdiag}).}

\end{example}

In the remainder, we will often drop $\Question$, implicitly assuming it is given, and use $P(true)$ to stand for $P(\Question=true)$.


In order to empower agents to reason about probabilities and thus support forecasting, we need, in addition to 
\emph{pro/con} arguments as in QuAD frameworks, two new argument types:

\begin{itemize}
    \item 
\emph{proposal} arguments,
each about some forecast (and its underlying forecasting question); each proposal argument $\Proposal$ has a \emph{forecast} 
and, optionally, some supporting \emph{evidence}
; and 
\item \emph{amendment} arguments
, which 
\AX{suggest a modification to}
some forecast\AX{'s probability} by increasing or decreasing it, and are accordingly separated into 
disjoint classes of \emph{increase} and \emph{decrease} (amendment) arguments.\footnote{Note that 
we decline to include a third type of amendment argument 
for arguing that $\NewForecast$ is just right. This choice rests on the assumption that additional information always necessitates a change to $\NewForecast$, however granular that change may be. This does not restrict individual agents arguing about $\NewForecast$ from casting $\NewForecast$ as their own final forecast. However, rather than cohering their argumentation around $\NewForecast$, which we hypothesise would lead to high risk of groupthink~\cite{McNees1987}, agents are compelled to consider the impact of their amendment arguments on this more granular level.}
\end{itemize}

Note that amendment arguments are introduced specifically for arguing about a proposal argument, given that traditional QuAD pro/con 
arguments are of limited use when the goal is to judge the acceptability of a probability, and that in forecasting agents must not only decide \emph{if} they agree/disagree but also \emph{how} they agree/disagree (i.e. whether they believe 
the forecast is too low or too high considering, if available, 
the evidence). Amendment arguments, with their increase and decrease classes, provide for this.

\begin{example}
\label{ProposalExample}
A proposal argument $\Proposal$ in the Tokyo Olympics setting may comprise 
forecast: \emph{\AX{`}There is a 75\% chance that the Olympics will be cancelled/postponed to another year'}. It may also include 
evidence: \emph{`A new poll today shows that 80\% of the Japanese public want the Olympics to be cancelled. The Japanese government is likely to buckle under this pressure.'}
This argument may aim to prompt updating the earlier forecast in Example~\ref{FAFEx}.
A \emph{decrease} amendment argument may be  $\decarga$: \emph{`The International Olympic Committee and the Japanese government will ignore the views of the Japanese public'}. An \emph{increase} amendment argument may be $\incarga$: \emph{`Japan's increasingly popular opposition parties will leverage this to make an even stronger case for cancellation'}.
\end{example}

Intuitively, a proposal argument 
is the focal point of the argumentation. It typically suggests a new forecast to replace prior forecasts, argued on the basis of some new evidence (as in the earlier example). We will see that proposal arguments remain immutable through each debate (update framework), which takes place via amendment arguments and standard pro/con arguments.
Note that, wrt QuAD argument types, proposal arguments replace issues and amendment arguments replace answers, in that the former are driving the debates and the latter are the options up for debate. 
Note also that amendment arguments merely state a direction wrt $\NewForecast$ and do not contain any more information, such as \emph{how much} to alter $\NewForecast$ by. 
We will see that alteration can be determined by \emph{scoring} amendment arguments.

Proposal and amendment arguments, alongside pro/con arguments, form part of 
our novel update frameworks \BI{(Stage 2 of Figure \ref{fig:FAFdiag})}, defined as follows:

\begin{definition} An \emph{update framework} is a nonad 
⟨$\Proposal, \AmmArgs, \AttArgs, \SuppArgs, \RelsP, \RelsQ, \Agents, \Votes, \AgentForecast$⟩ such that: 

\item[$\bullet$] $\Proposal$ is a single proposal argument 
with \emph{forecast} $\NewForecast
$ and, optionally, \emph{evidence} $\mathcal{E}^\Proposal$ for this forecast;
\item[$\bullet$] $\AmmArgs = \IncArgs \cup \DecArgs$ is a finite set of \emph{amendment arguments} composed of  subsets $\IncArgs$ of \emph{increase} arguments  and 
$\DecArgs$ of 
\emph{decrease} arguments;
\item[$\bullet$] $\AttArgs$ is a finite set  
of \emph{con} arguments;
\item[$\bullet$] $\SuppArgs$ is a finite set 
of \emph{pro} arguments;
\item[$\bullet$] the sets $\{\Proposal\}$, $\IncArgs$, $\DecArgs$, $\AttArgs$ and $\SuppArgs$ are pairwise disjoint;
\item[$\bullet$] $\RelsP$ $\subseteq$ $\AmmArgs$ $\times$ $\{\Proposal\}$ is a directed acyclic 
binary relation 
between amendment arguments and the proposal argument (we may refer to this relation informally as `probabilistic');
\item[$\bullet$] $\RelsQ$ $\subseteq$ ($\AttArgs$ $\cup$ $\SuppArgs$) $\times$ ($\AmmArgs$ $\cup$ $\AttArgs$ $\cup$ $\SuppArgs$) is a directed acyclic, 
binary relation 
\FTn{from} pro/con arguments 
\FTn{to} amendment\FTn{/pro/con arguments} (we may refer to this relation informally as `argumentative');

\item[$\bullet$] $\Agents = \{ \agent_1, \ldots, \agent_n \}$ is a finite set of \emph{agents} $(n >1$);
\item[$\bullet$] $\Votes$ : $\Agents$ $\times$ ($\AttArgs$ $\cup$ $\SuppArgs$) $\rightarrow$ [0, 1] is a total function such that $\Votes(\agent,\arg)$ is the \emph{vote} of agent $\agent\in\Agents$ on argument $\arg \in \AttArgs \cup \SuppArgs$; with an abuse of notation, we let $\Votes_\agent$ : ($\AttArgs$ $\cup$ $\SuppArgs$) $\rightarrow [0, 1]$ represent the votes of a \emph{single} agent $\agent\in\Agents$, e.g. $\Votes_\agent(\arg) = \Votes(\agent,\arg)$; 
\item[$\bullet$] $\AgentForecast = \{ \AgentForecast_{\agent_1}, \ldots, \AgentForecast_{\agent_n} \}$ is 
such that $\AgentForecast_{\agent_i} 
$, where $i \in \{ 1, \ldots n \}$, is the \emph{forecast} of agent $\agent_i\in\Agents$.
\end{definition}

\BIn{Note that pro \AX{(}con\AX{)} arguments can be seen as supporting (attacking, \resp) other arguments via $\Rels$, as in the case of conventional QuAD frameworks~\cite{Baroni2015}.}

\begin{example}
\label{eg:tokyo}
A possible update framework 
in our running setting may include $\Proposal$ as in Example~\ref{ProposalExample} as well as (see Table \ref{table:tokyo}) $\AmmArgs=\{\decarga, \decargb, \incarga\}$, $\AttArgs=\{\attarga, \attargb, \attargc\}$, $\SuppArgs=\{\supparga, \suppargb\}$, $\RelsP=\{(\decarga, \Proposal)$, $(\decargb, \Proposal), (\incargc, \Proposal)\}$, and $\Rels=\{(\attarga, \decarga), (\attargb, \decarga), (\attargc, \incarga)$, $(\supparga, \decargb),$ $ (\suppargb, \incarga)\}$
. Figure \ref{fig:tokyo} gives a graphical representation of 
these arguments and relations.
\BIn{Assuming $\Agents=\{alice, bob, charlie\}$, $\Votes$ may be such that $\AX{\Votes_{alice}(\attarga)} = 1$, $\AX{\Votes_{bob}(\supparga)} = 0$, and so on.}
\end{example}

\begin{table}[t]
\begin{tabular}{p{0.7cm}p{6.7cm}}
\hline
&
  Content \\ \hline
$\Proposal$ &
  `A new poll today shows that 80\% of the Japanese public want the Olympics to be cancelled owing to COVID-19, and the Japanese government is likely to buckle under this pressure ($\mathcal{E}^\Proposal)$. Thus, there is a 75\% chance that the Olympics will be cancelled/postponed to another year' ($\NewForecast$). \\
$\decarga$ &
  `The International Olympic Committee and the Japanese government will ignore the views of the Japanese public'. \\
$\decargb$ &
  `This poll comes from an unreliable source.' \vspace{2mm}\\
  $\incarga$ &
  `Japan's increasingly popular opposition parties will leverage this to make an even stronger case for cancellation.' \\ 
$\attarga$ &
  `The IOC is bluffing - people are dying, Japan is experiencing a strike. They will not go ahead with the games if there is a risk of mass death.' \\ 
$\attargb$ &
  `The Japanese government may renege on its commitment to the IOC, and use legislative or immigration levers to block the event.' \\
$\attargc$ &
  `Japan's government has sustained a high-approval rating in the last year and is strong enough to ward off opposition attacks.' \\
$\supparga$ &
  `This pollster has a track record of failure on Japanese domestic issues.' \\
$\suppargb$ &
  `Rising anti-government sentiment on Japanese Twitter indicates that voters may be receptive to such arguments.' \\ \hline
\end{tabular}
\caption {
Arguments in the update framework 
in Example~\ref{eg:tokyo}.}
\label{table:tokyo}
\end{table}

\begin{figure}[t]
\centering
\includegraphics[width=0.9\linewidth]{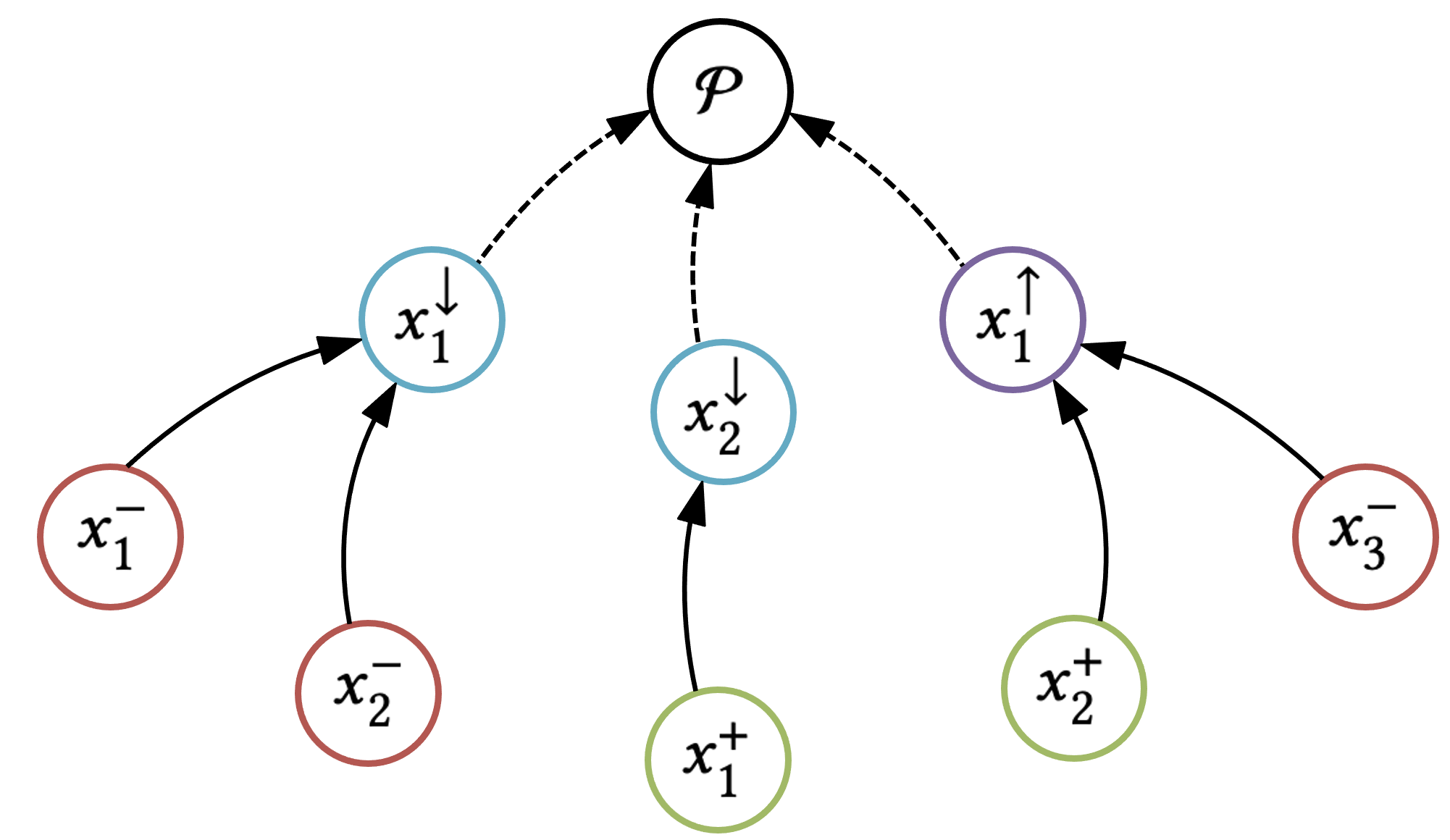}
\centering
\caption {\BIn{A graphical representation of arguments and relations in the update framework 
from Example~\ref{eg:tokyo}. Nodes 
represent proposal ($\Proposal$), increase ($\uparrow$), decrease ($\downarrow$), pro ($+$) and con ($-$)
arguments, while \FTn{dashed/solid} edges indicate
, \resp, the $\RelsP$/$\Rels$ relations. 
} 
}
\label{fig:tokyo}
\end{figure}

Several considerations about update frameworks are in order.
Firstly, they represent `stratified' debates: graphically, they can be represented as trees with 
the proposal argument as root,  amendment arguments as children of the root, and pro/con arguments forming the lower layers, as shown in Figure \ref{fig:tokyo}.
 This tree structure serves to focus argumentation towards the proposal  (i.e. the probability and, if available, evidence) it puts forward.
Second, we have chosen to impose a `structure' on proposal arguments, whereby their forecast is distinct from their (optional) evidence. Here the forecast has special primacy over the evidence, because forecasts are the vital reference point and the drivers of debates in FAFs. They are, accordingly, both mandatory and required to `stand out' to participating agents. In the spirit of abstract argumentation \cite{Dung1995}, we nonetheless treat all arguments, including proposal arguments, as `abstract', and focus on relations between them rather between their components. In practice, therefore, amendment arguments may relate to a proposal argument's forecast but also, if present, to its evidence. We opt for this abstract view on the assumption that the flexibility of this approach better suits judgmental forecasting, which has a diversity of use cases (e.g. including politics, economics and sport) where different argumentative approaches may be deployed (i.e. quantitative, qualitative, directly attacking amendment nodes or raising alternative POVs) and wherein forecasters may lack even a basic knowledge of argumentation.
We leave the study of structured variants of our framework (e.g. see overview in \cite{structArg})  to future work: these may consider finer-grained representations of all arguments in terms of different components, and finer-grained notions of relations between components, rather than full arguments. Third, in update frameworks, voting is restricted to pro/con arguments. Preventing agents from voting directly on amendment arguments mitigates against the risk of arbitrary judgements: agents cannot make off-the-cuff estimations but can only express their beliefs via (pro/con) argumentation, thus
ensuring a more rigorous process of appraisal for the proposal and amendment arguments.   
Note that rather than facilitating voting on arguments using a two-valued perspective (i.e. positive/negative)
or a three-valued perspective (i.e. positive/negative/neutral), $\Votes$ allows agents to cast more granular judgements of (pro/con) argument acceptability, the need for which has been highlighted in the literature \cite{Mellers2015}. 
Finally, although we envisage that arguments of all types are put forward by agents during debates, we do not capture this mapping in update frameworks. Thus, we do not capture who put forward which arguments, but instead only use votes to encode and understand agents' views. This enables more nuanced reasoning and full engagement on the part of agents with alternative viewpoints (i.e. an agent may freely argue both for and against a point before taking an explicit view with their voting). Such conditions are essential in a healthy forecasting debate \cite{Landeta2011,Mellers2015}.

In the remainder of this paper, with an abuse of notation, we often use $\NewForecast$ to denote, specifically,  the probability advocated in $\NewForecast$ (e.g. 0.75 in Example \ref{ProposalExample}).

\section{Aggregating Rational 
Forecasts 
}\label{sec:forecasting}

In this section we 
formally introduce (in \AX{§}\ref{subsec:rationality}) our notion of rationality and discuss how it may be used to identify\BI{, and subsequently `block',} undesirable behaviour in forecasters. We then define (in \AX{§}\ref{subsec:aggregation}) a method for calculating a revised forecast \BI{(Stage 2c of Figure \ref{fig:FAFdiag})}, which aggregates the views of all agents in the update framework, whilst optimising their overall forecasting accuracy.

\subsection{Rationality}\label{subsec:rationality}

Characterising an agent’s view as irrational offers opportunities to refine the accuracy of their forecast (and thus the overall aggregated group forecast). Our definition of rationality is inspired by, but goes beyond, that of QuAD-V \cite{Rago2017}, which was introduced for the e-polling context. Whilst update frameworks eventually produce a single aggregated forecast on the basis of group deliberation, each agent is first evaluated for their rationality on an individual basis. Thus, as in QuAD-V, in order to define rationality for individual agents, we first reduce frameworks to \emph{delegate frameworks} for each agent, which are the restriction of update frameworks 
to a single agent. 

\begin{definition}
A \emph{delegate framework} for an agent $\agent$ is $\update_{\agent} =$ ⟨$\Proposal, \AmmArgs, \AttArgs, \SuppArgs, \RelsP, \RelsQ, \agent, \Votes_{\agent}, \AgentForecast_{\agent}$⟩.
\end{definition} 

Note that all arguments in an update framework are included in each agent's delegate framework, but only the agent's votes and forecast are carried over.

Recognising the irrationality of an agent requires 
comparing the agent's forecast against (an aggregation of) their opinions on the amendment arguments and, by extension, on the proposal argument.
To this end, we evaluate the different parts of the update framework as follows. We use the DF-QuAD algorithm \cite{Rago2016} to score each amendment argument for the agent, in the context of the pro/con arguments `linked' to the amendment argument, using $\Rels$, in the context of the agent's delegate framework. We refer to the DF-QuAD score function as $\SF$.
This requires a choice of base scores for amendment arguments as well as pro/con arguments. 
We assume the same base score $\BS(\arg)=0.5$ for all $\arg \in \AmmArgs$;
in contrast, the base score of pro/con arguments is a result of the votes they received from the agent, in the spirit of QuAD-V \cite{Rago2017}.
The intuition behind assigning a neutral (0.5) base score for amendment arguments is that an agent's estimation of their strength from the outset would be susceptible to bias and inaccuracy.  
Once each amendment argument has been scored
(using $\SF$) for the agent, we aggregate the scores of all amendment arguments (for the same agent) to 
to calculate the agent's \emph{confidence score} in the proposal argument (which underpins our rationality constraints), by weighting the mean average strength of this argument's increase amendment relations against that of the set of decrease amendment relations:

\begin{definition}\label{def:conscore} 
Given a delegate framework $\update_{\agent}$ = ⟨$\Proposal$, $\AmmArgs$, $\AttArgs$, $\SuppArgs$, $\RelsP$, $\RelsQ$, $\agent$, $\Votes_{\agent}$, $\AgentForecast_{\agent}$⟩ 
, let 
$\IncArgs = \{ \incarga,  \incargb, \ldots , \incarg_i \}$ and  $\DecArgs = \{ \decarga,  \decargb, \ldots , \decarg_j \}$. 
Then,
$\agent$'s \emph{confidence score} is as follows:  
\begin{align}
&\text{if } i\neq0, j\neq0: \quad \Confidence_{\agent}(\Proposal) = \frac{1}{i} \sum_{k=1}^{i} \SF(\incarg_k) - \frac{1}{j} \sum_{l=1}^{j} \SF(\decarg_l); \nonumber \\
&\text{if } i\neq0, j=0: \quad \Confidence_{\agent}(\Proposal) = \frac{1}{i} \sum_{k=1}^{i} \SF(\incarg_k); \nonumber \\
&\text{if } i=0, j\neq0: \quad \Confidence_{\agent}(\Proposal) = - \frac{1}{j} \sum_{l=1}^{i} \SF(\decarg_l); \nonumber \\
&\text{if } i=0, j=0: \quad \Confidence_{\agent}(\Proposal) = 0. \nonumber
\end{align}
\end{definition} 

Note that $\Confidence_{\agent}(\Proposal) \in [-1,1]$, which denotes the overall views of the agent on the forecast $\NewForecast$ (i.e. as to whether it should be \emph{increased} or \emph{decreased}, and how far). A negative (positive) $\Confidence_{\agent}(\Proposal)$ indicates that an agent believes that $\NewForecast$ should be amended down (up, \resp). 
The size of $\Confidence_{\agent}(\Proposal)$ reflects the degree of the agent's certainty in either direction. 
In turn, we can constrain an agent's forecast $\AgentForecast_\agent$ if it contradicts this belief 
as follows.

\begin{definition}\label{def:irrationality}
Given a delegate framework $\update_{\agent}$ = ⟨$\Proposal$, $\AmmArgs$, $\AttArgs$, $\SuppArgs$, $\RelsP$, $\RelsQ$, $\agent$, $\Votes_{\agent}$, $\AgentForecast_{\agent}$⟩
, $\agent$’s forecast $\AgentForecast_\agent$ is \emph{strictly rational} (wrt $\update_{\agent}$) iff:
\begin{align}
if\;\Confidence_{\agent}(\Proposal) < 0\;then\; \AgentForecast_\agent < \NewForecast 
\nonumber \\
if\;\Confidence_{\agent}(\Proposal) > 0\;then\; \AgentForecast_\agent > \NewForecast 
\nonumber \\
\centering
\mid\Confidence_{\agent}(\Proposal)\mid \geq \frac{\mid\NewForecast - \AgentForecast_\agent\mid}{\NewForecast} 
\nonumber
\end{align}
\end{definition} 






Hereafter, we refer to forecasts which violate the first two constraints as, \resp, \emph{irrational increase} and \emph{irrational decrease} forecasts, and to forecasts which violate the final constraint as \emph{irrational scale} forecasts.

This definition of rationality preserves the integrity of group forecast in two ways. First, it prevents agents from forecasting against their beliefs: an agent cannot increase $\NewForecast$ if $\Confidence_{\agent}(\Proposal) < 0$ 
and an agent cannot decrease $\NewForecast$ if $\Confidence_{\agent}(\Proposal) > 0$
;
further, it ensures that agents cannot make forecasts disproportionate to their confidence score -- \emph{how far} an agent $\agent$ deviates from the proposed change $\NewForecast$ is restricted by $\Confidence_{\agent}(\Proposal)$; 
finally, an agent must have $\Confidence_{\agent}(\Proposal)$ greater than or equal to the relative change to $\NewForecast$ denoted in their forecast $\AgentForecast_\agent$
. 
Note that the \emph{irrational scale} constraint deals with just one direction of proportionality (i.e. providing only a maximum threshold for $\AgentForecast_\agent$'s deviation from $\NewForecast$, but no minimum threshold). Here, we avoid bidirectional proportionality on the grounds that such a constraint would impose an arbitrary notion of arguments' `impact' on agents. An agent may have a very high $\Confidence_{\agent}(\Proposal)$, indicating 
\FT{their} belief that $\NewForecast$ is too low, but \AX{may}, we suggest, rationally choose to increase $\NewForecast$ by only a small amount (e.g. if, despite 
\FT{their} general agreement with the arguments, 
\FT{they} believe the overall issue at stake in $\Proposal$ to be minor or low impact to the overall forecasting question). Our definition of rationality, which relies on notions of argument strength derived from DF-QuAD, thus informs but does not wholly dictate agents' forecasting, affording them considerable freedom. We leave alternative, stricter definitions of rationality, which may derive from probabilistic conceptions of argument strength, to future work.

\begin{example}
Consider our running Tokyo Olympics example, with the same 
arguments and relations from Example \ref{eg:tokyo} and an agent \BIn{$alice$} with a confidence score \BIn{$\Confidence_{alice}(\Proposal) = -0.5$}. From this we know that \BIn{$alice$} believes that the suggested 
$\NewForecast$ in the proposal argument $\Proposal$ 
should be decreased. 
Then, under our definition of rationality, \BIn{$alice$'s} forecast \BIn{$\AgentForecast_{alice}$} is `rational' if it decreases $\NewForecast$ by up to 50\%. 
\end{example} 

If an agent's forecast $\AgentForecast_\agent$ violates these rationality constraints then \BI{it is `blocked'} and the agent is prompted to return to the argumentation graph. From here, they may carry out one or more of the following actions to render their forecast rational: 

a. Revise their forecast; 

b. Revise their votes on arguments; 

c. Add new arguments 
(and vote on 
them).

Whilst a) and b) occur on an agent-by-agent basis, confined to each delegate framework, c) affects the shared update framework and requires special consideration.
Each time new \AX{arguments} 
are added to the shared graph, every agent must vote on 
\AX{them}, even if they have already made a rational forecast. In certain cases, after an agent has voted on a new argument, it is possible that their rational forecast is made irrational. In this instance, the agent must resolve their irrationality via the steps above. In this way, the update framework can be refined on an iterative basis until the graph is no longer being modified and all agents' forecasts are rational. At this stage, the update framework has reached a stable state and the agents $\Agents$ are collectively rational:

\begin{definition} Given an update framework $\update$ = ⟨$\Proposal$, $\AmmArgs$, $\AttArgs$, $\SuppArgs$, $\RelsP$, $\RelsQ$, $\Agents$, $\Votes$, $\AgentForecast$⟩, $\Agents$ is \emph{collectively rational} (wrt \emph{u}) iff $\forall \agent \in \Agents$, $\agent$ is individually rational (wrt the
delegate framework $\update_{\agent}$ = ⟨$\Proposal$, $\AmmArgs$, $\AttArgs$, $\SuppArgs$, $\RelsP$, $\RelsQ$, $\agent$, $\Votes_{\agent}$, $\AgentForecast_{\agent}$⟩).
\end{definition}

When $\Agents$ is collectively rational, the 
update framework $u$ becomes immutable and the aggregation (defined next) 
\AX{produces} a group forecast $\GroupForecast$, which becomes the 
new $\Forecast$.

\subsection{Aggregating Forecasts}\label{subsec:aggregation}

After all the agents have made a rational forecast, an aggregation function is applied to produce one collective forecast. One advantage of forecasting debates vis-a-vis 
\AX{the} many other forms of debate, is that a ground truth always exists -- an event either happens or does not.  This means that, over time and after enough FAF instantiations, data on the forecasting success of different agents can be amassed. In turn, the relative historical performance of forecasting agents can inform the aggregation of group forecasts. In update frameworks, a weighted aggregation function based on Brier Scoring \cite{Brier1950} is used, such that more accurate forecasting agents have greater influence over the final forecast. 
Brier Scores are a widely used criterion to measure the accuracy of probabilistic predictions, effectively gauging the distance between a forecaster's predictions and an outcome after it has(n't) happened, as follows. 

\begin{definition} \label{def:bscore}
Given an agent $\agent$, a non-empty series of forecasts $\AgentForecast_\agent(1), \ldots, \AgentForecast_\agent(\Number_{\agent})$ with corresponding actual outcomes $\Outcome_1, \ldots,$ $\Outcome_{\Number_{\agent}} \in \{true, false\}$ (where $\Number_{\agent}>0$ is the number  of forecasts $\agent$ has made in a non-empty sequence of as many update frameworks), $\agent$'s Brier Score $\Brier_{\agent} \in [0, 1]$ is as follows:
\begin{align}
\Brier_{\agent} = \frac{1}{\Number_{\agent}} \sum_{t=1}^{\Number_{\agent}} (\AgentForecast_\agent(t) - val(\Outcome_t))^2 \nonumber
\end{align}
where $val(\Outcome_t)=1$ if $\Outcome_t=true$, and 0 otherwise.
\end{definition} 


A Brier Score $\Brier$ is effectively the mean squared error used to gauge forecasting accuracy, where a low $\Brier$ indicates high accuracy and high $\Brier$ indicates low accuracy. This 
can be used in the update framework's aggregation function via the weighted arithmetic mean as follows. 
\AX{E}ach Brier Score is inverted to ensure that more (less, \resp) accurate forecasters have higher (lower, \resp) weighted influence\AX{s} on $\GroupForecast$: 

\begin{definition}\label{def:group} 
Given a set of agents $\Agents = \{\agent_1, \ldots,\agent_n\}$, 
their corresponding set of Brier Scores $\Brier = \{\Brier_{\agent_1}, \ldots,\Brier_{\agent_n}\}$ and 
their forecasts $\AgentForecast = \{\AgentForecast_{\agent_1}, \ldots,\AgentForecast_{\agent_n}\}$, and letting, for $i \!\!\in\!\! \{ 1, \ldots, n\}$, $w_{i} \!\!=\!\! 1-\Brier_{\agent_i}$, the \emph{group forecast} $\GroupForecast$ is 
as follows:
\begin{align}
&\text{if } \sum_{i=1}^{n}w_{i} \neq 0: &
&\GroupForecast =   \frac{\sum_{i=1}^{n}w_{i}\AgentForecast_{\agent_i}}{\sum_{i=1}^{n}w_{i}}; \nonumber \\
&\text{otherwise}: &
&\GroupForecast = 0. \nonumber
\end{align}
\end{definition}

This group forecast could be `activated' after a fixed number of debates (with the mean average used prior), when sufficient data has been collected on the accuracy of participating agents, or after a single debate, in the context of our general 
\emph{Forecasting Argumentation Frameworks}:

\begin{definition} A \emph{Forecasting Argumentation Framework} (FAF) is a triple ⟨$ \Forecast, \Update, \Time$⟩ such that: 

\item[$\bullet$] $\Forecast$ is a \emph{forecast}
;
\item[$\bullet$] $\Update$ is a finite, non-empty sequence of update frameworks  with \Forecast\ the forecast of the proposal argument in the first update framework in the sequence\AR{;} the forecast of each subsequent update framework is the group forecast of the previous update framework's agents' forecasts; 
\item[$\bullet$] $\Time$ is a preset time limit representing the lifetime of the FAF;
\item[$\bullet$]  each agent's forecast wrt the agent's delegate framework drawn from each update framework is strictly rational.
\end{definition}

\begin{example}
\BIn{Consider our running Tokyo Olympics example: the overall FAF may be composed of $\Forecast = 0.15$, update frameworks $\Update = \{ u_1, u_2, u_3 \}$ and time limit $\Time=14\ days$, where $u_3$ is the latest (and therefore the only open) update framework after, for example, four days.}
\end{example} 

\AX{T}he superforecasting literature explores a range of forecast aggregation algorithms: extremizing algorithms \cite{Baron2014}, variations on logistic \AX{and} 
Fourier $L_2E$ regression \cite{Cross2018}, with considerable success. 
\AX{T}hese approaches 
\AX{aim} 
at ensuring that less certain 
\AX{or less} accurate forecasts have a lesser influence over the final aggregated forecast. We believe that 
FAFs apply a more intuitive algorithm 
\AX{since} much of the `work' needed to bypass inaccurate and erroneous forecasting is 
\AX{expedited}
via argumentation.

\section{Properties}\label{sec:props}

We now undertake a theoretical analysis of FAFs by considering  mathematical properties they satisfy. Note that the properties of the DF-QuAD algorithm (see \cite{Rago2016}) hold (for the amendment and pro/con arguments) here. For brevity, we 
focus on novel properties unique to FAFs which relate to our new argument types. These properties focus on  aggregated group forecasts wrt a debate (update framework). They imply the two broad, and we posit, desirable, principles of \emph{balance} and \emph{unequal representation}.
We assume for this section a generic update framework $\update = $ ⟨$\Proposal$, $\AmmArgs$, $\AttArgs$, $\IncArgs$, $\RelsP$, $\RelsQ$, $\Agents$, $\Votes$, $\AgentForecast$⟩ with 
group forecast $\GroupForecast$.

\paragraph{Balance.}

The intuition for these properties is that 
differences between 
$\GroupForecast$ and 
$\NewForecast$ correspond to imbalances between the 
\emph{increase} and \emph{decrease} amendment arguments. 


The first result states that 
$\GroupForecast$ only differs from 
$\NewForecast$ if $\NewForecast$ is the dialectical target of amendment arguments.

\begin{proposition} \label{prop:balance1}
If $\AmmArgs\!\!=\!\!\emptyset$ ($\DecArgs\!\!=\!\!\emptyset$ and $\IncArgs\!\!=\!\!\emptyset$), then $\GroupForecast\!\!=\!\!\NewForecast$. 
\end{proposition}

\begin{proof}
\AX{If $\DecArgs\!\!=\!\!\emptyset$ and $\IncArgs\!\!=\!\!\emptyset$ then $\forall \agent \!\in\! \Agents$, 
$\Confidence_{\agent}(\Proposal)\!=\!0$ by \Definition~\ref{def:conscore} and $\AgentForecast_\agent=\NewForecast$ by \Definition~\ref{def:irrationality}. 
Then, $\GroupForecast=\NewForecast$ by \Definition~\ref{def:group}.}
\end{proof}

\AX{T}his simple proposition conveys an important property for forecasting:
for an agent to put forward a different forecast, amendment arguments must have been introduced.

\begin{example}
In the Olympics setting, the group of agents could only forecast higher or lower than the proposed forecast $\NewForecast$ after the addition of at least one of \AX{the} amendment arguments $\decarga$, $\decargb$ or $\incarga$.
\end{example}


In the absence of 
increase \FTn{(decrease)} amendment arguments, if there 
are decrease \FTn{(increase, \resp)} amendment arguments, then 
$\GroupForecast$ is not higher \FTn{(lower, \resp)} than $\NewForecast$.

\begin{proposition}\label{prop:balance2}
If $\DecArgs\neq\emptyset$ and $\IncArgs=\emptyset$, then $\GroupForecast \leq\NewForecast$.
\FTn{\label{balance3prop} 
If $\DecArgs=\emptyset$ and $\IncArgs\neq\emptyset$, then $\GroupForecast\geq\NewForecast$.}
\end{proposition}

\begin{proof}
\AX{If $\DecArgs\!\! \neq \!\!\emptyset$ and $\IncArgs\!\!=\!\!\emptyset$ then $\forall \agent \!\in\! \Agents$, $\Confidence_{\agent}(\Proposal)\!\leq\!0$ by \Definition~\ref{def:conscore} and then $\AgentForecast_\agent\!\leq\!\NewForecast$ by \Definition~\ref{def:irrationality}. 
Then, by \Definition~\ref{def:group}, $\GroupForecast\!\leq\!\NewForecast$.
If $\DecArgs\!\!=\!\!\emptyset$ and $\IncArgs\!\!\neq\!\!\emptyset$ then $\forall \agent \!\in\! \Agents$, $\Confidence_{\agent}(\Proposal)\!\geq\!0$ by \Definition~\ref{def:conscore} and then 
$\AgentForecast_\agent\!\geq\!\NewForecast$ by \Definition~\ref{def:irrationality}. Then, by \Definition~\ref{def:group}, $\GroupForecast\!\geq\!\NewForecast$.}
\end{proof}

This proposition demonstrates that, if a decrease \BIn{(increase)} amendment argument has an effect on proposal arguments, it can only be as its name implies.

\begin{example}
\BIn{In the Olympics setting, the 
agents could not forecast higher than the proposed forecast $\NewForecast$ if either of the decrease amendment arguments $\decarga$ or $\decargb$ 
\AX{had} been added, but the increase argument $\incarga$ 
\AX{had} not. Likewise, \AX{the} 
agents could not forecast lower than 
$\NewForecast$ if 
$\incarga$ 
\AX{had} been added, but 
\AX{neither} of 
$\decarga$ or $\decargb$ \AX{had}
.}
\end{example}






If 
$\GroupForecast$ is lower \BIn{(higher)} than 
$\NewForecast$, 
there is 
at least one decrease \BIn{(increase, resp.)} argument.

\begin{proposition} \label{prop:balance4}
If $\GroupForecast<\NewForecast$, then $\DecArgs\neq\emptyset$. \BIn{If $\GroupForecast>\NewForecast$, then $\IncArgs\neq\emptyset$.} 
\end{proposition}

\begin{proof}
\AX{
If $\GroupForecast<\NewForecast$ then, by  \Definition~\ref{def:group}, $\exists \agent \in \Agents$ where $\AgentForecast_{\agent}<\NewForecast$, for which it holds from \Definition~\ref{def:irrationality} that $\Confidence_{\agent}(\Proposal)<0$. 
Then, irrespective of $\IncArgs$, $\DecArgs\neq\emptyset$. If $\GroupForecast>\NewForecast$ then, by \Definition~\ref{def:group}, $\exists \agent \in \Agents$ where $\AgentForecast_{\agent}>\NewForecast$, for which it holds from \Definition~\ref{def:irrationality} that $\Confidence_{\agent}(\Proposal)>0$. Then, irrespective of \BIn{$\DecArgs$, $\IncArgs\neq\emptyset$}.
}
\end{proof}

We can see here that the only way an agent can decrease \BIn{(increase)} the forecast is 
\FT{by adding} decrease \BIn{(increase, resp.)} arguments, ensuring the debate is structured as 
\FT{intended}.

\begin{example}
\BIn{In the Olympics setting, the group of agents could only produce a group forecast $\GroupForecast$ lower than 
$\NewForecast$ due to the presence of 
\emph{decrease} amendment arguments $\decarga$ or $\decargb$. Likewise, the group of agents could only produce a 
$\GroupForecast$ higher than 
$\NewForecast$ due to the presence of 
$\incarga$.}
\end{example}


\paragraph{Unequal representation.}

AFs exhibit instances of unequal representation in the final voting process. In formulating the following properties, we distinguish between two forms of unequal representation. First, \emph{dictatorship}, where a single agent dictates 
$\GroupForecast$ with no input from other agents. Second, \emph{pure oligarchy}, where a group of agents dictates 
$\GroupForecast$ with no input from other agents outside the group. 
In the forecasting setting, these 
properties are desirable as they guarantee higher accuracy 
\AX{from} the group forecast $\GroupForecast$.

An agent with a forecasting record of \emph{some} accuracy exercises \emph{dictatorship} over the group forecast $\GroupForecast$, if the rest of the participating \AX{agents} 
have a record of total inaccuracy.

\begin{proposition}\label{prop:dictatorship}
If $\agent_d\in\Agents$ has a Brier score $\Brier_{\agent_d}<1$ and $\forall \agent_z\in\Agents \setminus \{\agent_d$\}, $\Brier_{\agent_z} = 1$, then $\GroupForecast=\AgentForecast_{\agent_d}$.
\end{proposition}

\begin{proof}
\AX{
By \Definition~\ref{def:group}: if  $\Brier_{\agent_z} \!\!\!=\!\! 1$ $\forall \agent_z\!\in\!\Agents \!\setminus\! \{\!\agent_d\!\}$, then $w_{\agent_z}\!\!\!=\!0$; and if $\Brier_{\agent_d}\!\!<\!\!1$, then $w_{\agent_d}\!\!>\!\!0$
. Then, again by \Definition~\ref{def:group}, $\AgentForecast_{\agent_d}$ is weighted at 100\% and $\AgentForecast_{\agent_z}$ is weighted at 0\% so $\GroupForecast\!\!=\!\!\AgentForecast_{\agent_d}$.
}
\end{proof}

This proposition demonstrates how we will disregard agents with total inaccuracy, even in 
\FT{the} extreme case where we allow one (more accurate) agent to dictate the forecast.


\begin{example}
\BIn{In the running example, if \AX{alice, bob and charlie have Brier scores of 0.5, 1 and 1, \resp, bob's and charlie's forecasts have} no impact on $\GroupForecast$, whilst \AX{alice's} forecast becomes the group forecast $\GroupForecast$.}
\end{example}

A group of agents with a forecasting record of $some$ accuracy exercises \emph{pure oligarchy} over 
$\GroupForecast$ if the rest of the 
\AX{agents} all have a record of total inaccuracy.

\begin{proposition}\label{oligarchytotalprop}
Let $\Agents = \Agents_o \cup \Agents_z$ where $\Agents_o \cap \Agents_z = \emptyset$, $\Brier_{\agent_o}<1$ $\forall \agent_o \in \Agents_o$ and $\Brier_{\agent_z}=1$ $\forall \agent_z \in \Agents_z$. Then, $\AgentForecast_{\agent_o}$ is weighted at $>0\%$ 
and $\AgentForecast_{\agent_z}$ is weighted at 0\%
.
\end{proposition}

\begin{proof}
\AX{
By \Definition~\ref{def:group}: if $\Brier_{\agent_z} = 1$ $\forall \agent_z\in\Agents_z$, then $w_{\agent_z}=0$; and if $\Brier_{\agent_o}<1$ $\forall \agent_o\in\Agents_o$, then $w_{\agent_o}>0$. Then, again by \Definition~\ref{def:group}, $\AgentForecast_{\agent_o}$ is weighted at $> 0\%$ and $\AgentForecast_{\agent_z}$ is weighted at $0\%$.
}
\end{proof}

This proposition extends the behaviour from Proposition \ref{prop:dictatorship} to the (more desirable) case where fewer agents have a record of total inaccuracy.

\begin{example}
\BIn{In the running example, if \AX{alice, bob and charlie have Brier scores of 1, 0.2 and 0.6, \resp, alice's forecast} has no impact on $\GroupForecast$, whilst \AX{bob and charlie's} aggregated forecast becomes the group forecast $\GroupForecast$.}
\end{example}




\section{Evaluation}\label{sec:experiments}


\BI{We conducted an experiment using a dataset obtained from the `Superforecasting' project, Good Judgment Inc \cite{GJInc}, to simulate four past forecasting debates in FAFs. This dataset contained 1770 datapoints (698 `forecasts' and 1072 `comments') posted by 242 anonymised users with a range of expertise. The original debates had occurred on the 
publicly available group forecasting platform, the Good Judgment Open (GJO)\footnote{https://www.gjopen.com/}, providing a suitable baseline against which to compare FAFs' 
accuracy
.}




\BI{For the experiment, we used a 
prototype implementation of FAFs in the form of the publicly available web platform called \emph{Arg\&Forecast} (see \cite{Irwin2022} for an introduction to the platform and an additional human 
experiment with FAFs). Python's Gensim topic modelling library \cite{rehurek2011gensim} was used 
to separate the datapoints for each debate into contextual-temporal groups that could form update frameworks.} In each update framework the proposal forecast was set to the mean average of forecasts made in the update framework window and each argument appeared only once. Gensim was further used to simulate voting, matching users to specific arguments they (dis)approved of. Some 4,700 votes 
\AX{were then}
generated with a three-valued system (where votes were taken from  
\{0,0.5,1\}) to ensure consistency: if a user voiced approval for an argument in the debate time window, their vote for the corresponding argument(s) was set to 1; disapproval for an argument led to a vote of 0, and (in the most common case) if a user did not mention an argument at all, their vote for the corresponding argument(s) defaulted to 0.5. 

With the views of all participating users wrt the proposal argument encoded in each update framework's votes, forecasts could then be simulated. If a forecast was irrational, violating any of the three constraints in \Definition~\ref{def:irrationality} (referred to 
\AX{in the following}
as \emph{increase}, \emph{decrease} and \emph{scale}, \resp), it was blocked and, to mimic real life use, an automatic `follow up' forecast was made. The `follow up' forecast would be the closest possible prediction (to their original choice) a user could make whilst remaining `rational'. 

\BI{Note that evaluation of the aggregation function described in \AX{§}\ref{subsec:aggregation} was outside this experiment, since the past forecasting accuracy of the dataset's 242 anonymised users was unavailable. Instead, we used \AX{the} mean average whilst adopting the GJO's method for scoring the accuracy of a user and/or group over the lifetime of the question \cite{roesch_2015}. This meant calculating a daily forecast and daily Brier score 
for each user, for every day of the question. After users made their first rational forecast, that forecast became their `daily forecast' 
until it was updated with a new forecast. Average and range of daily Brier scores 
allowed reliable comparison between (individual and aggregated) performance of the GJO versus the FAF implementation.} 

\begin{table}[t]
\begin{tabular}{@{}llllll@{}}
\toprule
Q & Group $\Brier$         & $min(\Brier)$  & $max(\Brier)$           \\ \midrule
Q1       & 0.1013 (0.1187) & 0.0214 (0) & 0.4054 (1)        \\
Q2       & 0.216 (0.1741)   & 0 (0) & 0.3853 (1)            \\
Q3       & 0.01206 (0.0227)   & 0.0003 (0) & 0.0942 (0.8281)  \\
Q4       & 0.5263 (0.5518)   & 0 (0) & 0.71 (1)             \\ \midrule
\textbf{All} & \textbf{0.2039 (0.217)} & \textbf{0 (0)} & \textbf{1 (1)} \\ \bottomrule
\end{tabular}
\caption{The accuracy 
of the platform group versus control, where \AX{`}Group $\Brier$\AX{'} is the aggregated (mean) Brier score, `$min(\Brier)$' is the lowest individual Brier score and `$max(\Brier)$' is the highest individual Brier score. Q1-Q4 indicate the four simulated  debates.}
\label{accuracyExp1}
\end{table}
\begin{table}[t]
\begin{tabular}{llllll}
\hline
\multirow{2}{*}{Q} &
  \multirow{2}{*}{$\overline{\Confidence}$} &
  \multirow{2}{*}{Forecasts} &
  \multicolumn{3}{c}{Irrational Forecasts} \\ \cline{4-6} 
 &
   &
   &
  \multicolumn{1}{c}{\emph{Increase} 
  } \!\!\!\!
  &
  \multicolumn{1}{c}{\emph{Decrease} 
  } \!\!\!\!
  &
  \multicolumn{1}{c}{\emph{Scale} 
  }\!\! \!\!
  \\ \hline
Q1  & -0.0418 & 366 & 63  & 101 & 170 \\
Q2  & 0.1827  & 84  & 11  & 15  & 34  \\
Q3  & -0.4393 & 164 & 53  & 0   & 86  \\
Q4  & 0.3664  & 84  & 4   & 19  & 15  \\ \hline
All & -0.0891 & 698 & 131 & 135 & 305 \\ \hline
\end{tabular}
\caption{Auxiliary results from \FT{the experiment}, where 
$\overline{\Confidence}$ is the average confidence score, `Forecasts' is number of forecasts made in each question and `Irrational Forecasts' the number in each question which violated each constraint in §\ref{subsec:rationality}.}
\label{exp1auxinfo}
\end{table}

\paragraph{Results.} 

As Table \ref{accuracyExp1} demonstrates, simulating forecasting debates from GJO in \emph{Arg\&Forecast} led to predictive accuracy improvements in three of the four debates. \BIn{This is reflected in these debates by a substantial reduction in Brier scores versus control.} The greatest accuracy improvement in absolute terms was in Q4, which saw a Brier score decrease of 0.0255. In relative terms, Brier score decreases ranged from 5\% (Q4) to 47\% (Q3). \BIn{The average Brier score decrease was 33\%, representing a significant improvement in forecasting accuracy across the board}. \BIn{Table \ref{exp1auxinfo} demonstrates how 
\AX{our}
rationality constraints drove forward this improvement}. 82\% of forecasts made across the four debates were classified as irrational \BIn{and subsequently moderated with a rational `follow up' forecast}. Notably, there were more \emph{irrational scale} forecasts 
than 
\emph{irrational increase} 
and \emph{irrational decrease} forecasts 
combined. These results demonstrate how argumentation-based rationality constraints can play an active role in facilitating higher forecasting accuracy, signalling the early promise of FAFs.

\section{Conclusions}\label{sec:conclusions}

We have introduced the Forecasting Argumentation Framework (FAF), a multi-agent argumentation framework which supports forecasting debates and probability estimates. FAFs are composite argumentation frameworks, comprised of multiple non-concurrent update frameworks which themselves depend on three new argument types and a novel definition of rationality for the forecasting context. Our theoretical and empirical evaluation demonstrates the potential of FAFs, namely in increasing forecasting accuracy, holding intuitive 
properties, identifying irrational 
behaviour and driving higher engagement with the forecasting question (more arguments and responses, and more forecasts in the user study). These strengths align with requirements set out by previous research in the field of judgmental forecasting.

There 
\AX{is} a multitude of possible directions for future work
. First, FAFs are equipped to deal only with two-valued outcomes but, given the prevalence of forecasting issues with multi-valued outcomes (e.g. `Who will win the next UK election?'), expanding their capability 
would add value. Second, further work may focus on the rationality constraints, 
e.g. by introducing additional parameters to adjust their strictness, or 
\AX{by implementing}
alternative interpretations of rationality
. Third, 
future work could explore constraining agents' argumentation. This could involve using past Brier scores to limit the quantity or strength of agents' arguments and also 
to give them greater leeway wrt the rationality constraints.
\FTn{Fourth, 
our method relies upon 
 acyclic graphs:  we believe that they are intuitive for users  and note that all Good Judgment Open debates were acyclic; nonetheless, the inclusion of cyclic relations (e.g. to allow 
 \AX{con} arguments that attack each other) could expand the scope of the argumentative reasoning in 
 \AX{in FAFs.}
} 
Finally, there is an immediate need for larger scale human 
experiments.

\newpage

\section*{Acknowledgements}

The authors would like to thank Prof. Anthony Hunter for his helpful contributions to discussions in the build up to this work. \BIn{Special thanks, in addition, go to Prof. Philip E. Tetlock and the Good Judgment Project team for their warm cooperation and for providing datasets for the experiments.}
\AX{Finally, the authors would like to thank the anonymous reviewers and meta-reviewer for their suggestions, which led to a significantly improved paper.}


\bibliographystyle{kr}
\bibliography{bib.bib}

\end{document}